\documentclass[pdflatex,sn-mathphys-num]{sn-jnl}

\usepackage{graphicx}%
\usepackage{multirow}%
\usepackage{amsmath,amssymb,amsfonts}%
\usepackage{amsthm}%
\usepackage{float}
\usepackage[scr=rsfso,cal=rsfso]{mathalfa} 
\let\mathcal\mathscr

\usepackage[export]{adjustbox}
\usepackage[title]{appendix}%
\usepackage{xcolor}%
\usepackage{textcomp}%
\usepackage{array} 
\usepackage{manyfoot}%
\usepackage{booktabs}%
\usepackage{adjustbox}
\usepackage{algorithm}%
\usepackage{algorithmicx}%
\usepackage{algpseudocode}%
\usepackage{listings}%

\usepackage{xurl} 

\usepackage{etoolbox}
\AtBeginDocument{\hypersetup{hypertexnames=false}}

\theoremstyle{thmstyleone}%
\newtheorem{theorem}{Theorem}
\theoremstyle{thmstyletwo}%

\theoremstyle{thmstylethree}%

\begin{document}

\title[Enhanced 3D Shape Analysis via Information Geometry]{Enhanced 3D Shape Analysis via Information Geomerty}

\author[1]{\fnm{Amit} \sur{Vishwakarma}}\email{amitvishwakarma.22@res.iist.ac.in}
\author*[1]{\fnm{K.S. Subrahamanian} \sur{Moosath}}\email{smoosath@iist.ac.in}

\affil[1]{\orgname{Indian Institute of Space Science and Technology}, \orgaddress{\city{Thiruvananthapuram}, \country{India}}}


\abstract{
Three-dimensional point clouds provide highly accurate digital representations of objects, essential for applications in computer graphics, photogrammetry, computer vision, and robotics. However, comparing point clouds faces significant challenges due to their unstructured nature and the complex geometry of the surfaces they represent. Traditional geometric metrics such as Hausdorff and Chamfer distances often fail to capture global statistical structure and exhibit sensitivity to outliers, while existing Kullback-Leibler (KL) divergence approximations for Gaussian Mixture Models can produce unbounded or numerically unstable values. This paper introduces an information geometric framework for 3D point cloud shape analysis by representing point clouds as Gaussian Mixture Models (GMMs) on a statistical manifold. We prove that the space of GMMs forms a statistical manifold and propose the Modified Symmetric Kullback-Leibler (MSKL) divergence with theoretically guaranteed upper and lower bounds, ensuring numerical stability for all GMM comparisons. Through comprehensive experiments on human pose discrimination (MPI-FAUST dataset) and animal shape comparison (G-PCD dataset), we demonstrate that MSKL provides stable and monotonically varying values that directly reflect geometric variation, outperforming traditional distances and existing KL approximations.
}




\maketitle

\section{Introduction}\label{sec1}

Three-dimensional point clouds provide highly accurate digital representations of objects which makes it important for the applications in computer graphics, photogrammetry, computer vision, construction, and remote sensing \cite{10.1016/j.aei.2019.02.007,chen2022stpls3d,akagic2022cv}. Recent work has addressed various geometric analysis tasks including robust symmetry detection in the presence of outliers \cite{Nagar2025}. Extracting meaningful information from point clouds is fundamental for shape analysis tasks such as pose estimation, deformation tracking, registration, retrieval and autoencoding\cite{kucak2022ransac, zhang2025dlpc, huang2021pcsurvey}. However, comparing point clouds faces unique challenges due to their unstructured nature and the complex geometry of the surfaces they represent. Unlike structured grid data with natural Euclidean metrics, point clouds lack a common coordinate system, making direct comparisons difficult. When comparing shapes sampled from 3D surfaces, the objective is to capture global geometric similarity rather than point-wise correspondence.

Classical geometric distances such as Hausdorff \cite{hausdorff2008gesammelte} and Chamfer \cite{yang2018} measure spatial proximity between point sets. While effective for certain tasks, they often fail to capture broader statistical structure, exhibit sensitivity to outliers and sampling irregularities, and do not characterize the global distribution of points. 

Earth Mover's Distance (EMD) \cite{rubner2000earth} computes similarity as the minimum cost of transforming one distribution into another, providing a transport-based metric. However, EMD requires solving an expensive optimization problem that scales poorly with point cloud size and often struggles to capture intrinsic geometric relationship between the shapes.

Recent works have explored probabilistic representations using GMMs for point cloud registration and comparison \cite{jian2005robust}. While these methods make use of statistical modeling, they generally rely on approximations of the Kullback–Leibler (KL) divergence and these approximations have some limitations:
\begin{itemize}
    \item KL Weighted-Average (KL$_{\text{WA}}$) \cite{1238387,hershey2007approximating} approximation provides a lower bound by averaging pairwise KL divergences between Gaussian components weighted by mixing coefficients. However, this approximation can significantly underestimate the true divergence, particularly when mixture components have poor overlap.
    \item KL Matching-Based (KL$_{\text{MB}}$) \cite{1238387,hershey2007approximating} approximation matches each component of one mixture to the closest component in another, yielding an upper bound. This matching strategy can overestimate divergence and is sensitive to the number and initialization of components.
    \item Furthermore, the standard KL divergence is asymmetric (${\text{KL}}(p||q) \neq {\text{KL}}(q||p)$) and \textbf{can be unbounded} when comparing GMMs, leading to numerical instabilities and undefined comparisons in practical scenarios.
\end{itemize}

This work develops an information geometric framework for stable shape similarity measurement by representing each point cloud as a GMM defined on a low-dimensional latent space. The framework includes preprocessing, extraction of local descriptors, manifold embedding via Isomap, and GMM fitting through the Expectation–Maximization (EM) algorithm. With this representation, the collection of all such GMMs forms a statistical manifold, enabling the use of divergence-based tools for shape analysis.
We introduce the Modified Symmetric KL (MSKL) divergence which addresses the instability issue of exisitng KL approaximations. The theoretical properties, including bounds and stability for MSKL are established in the subsequent sections. In addition, we consider two well-known approximations for KL-divergence on GMMs—the weighted-average ($KL_{\text{WA}}$) and matching-based (KL$_{\text{MB}}$) which provide complementary baselines frequently used in mixture model comparison.

The main contributions of this work are:
\begin{itemize}
    \item Established a rigorous mathematical framework demonstrating that point clouds represented as GMMs form a statistical manifold, enabling the use of information-geometric methods.
    \item Introduced the MSKL divergence and prove that it is well defined by establishing upper and lower bounds, providing theoretical guarantee for GMM-based shape comparison.
    \item Compared the performance of MSKL with symmetric KL-based approximations (SKL\(_\text{WA}\), SKL\(_\text{MB}\)) and classical geometric distances (Hausdorff, Chamfer).
\end{itemize}

To demonstrate the practical performance of the framework, we compare MSKL with Symmetric KL divergence (SKL), SKL\(_\text{WA}\), SKL\(_\text{MB}\) and the traditional distances  Hausdorff, and Chamfer across two shape categories: (i) human shapes from the MPI FAUST dataset \cite{mpi_faust_dataset} and (ii) animal shapes from the G-PCD dataset \cite{epfl_geometry_point_cloud_dataset}. These experiments highlight how statistical manifold representations capture both local geometry and global structure more effectively than metrics.

\section{Related Works}\label{sec:related}
In this section, we review existing methods for comparing point clouds, which can be broadly classified into two categories: metric-based distances and divergence-based statistical approaches.

\subsection{Metric-Based Approaches}
\paragraph{1. Hausdorff Distance.}
For two finite points $P,Q\in \mathbb{R}^{m}$ the Hausdorff distance \cite{hausdorff2008gesammelte} is defined as

\begin{equation}
D_H(P, Q) = \max\left\{\sup_{p \in P} \inf_{q \in Q} \| p - q \|, \sup_{q \in Q} \inf_{p \in P} \| p - q \|\right\}.
\end{equation}

This measure emphasizes the largest difference between two sets of 3D points, which makes it particularly sensitive to unusual points that do not fit in the pattern.

\paragraph{2. Chamfer Distance.}
The Chamfer distance \cite{yang2018} provides an alternative that averages point-to-point distances, thereby reducing the impact of outliers. It is widely used in 3D reconstruction and mesh generation.

\begin{equation}
Ch(P, Q) = \frac{1}{|P|} \sum_{p \in P} \min_{q \in Q} \| p - q \| + \frac{1}{|Q|} \sum_{q \in Q} \min_{p \in P} \| p - q \|.
\end{equation}

This method calculates the average of the squared distances between each point in one set to its nearest point in the other set, offering a balance between sensitivity to large discrepancies and outliers. In this method differences in point cloud shape and distribution is not taken care of.

\paragraph{3. Earth Mover's Distance.}
The Earth Mover's Distance (EMD) formulates point cloud comparison as an optimal transport problem:
\begin{equation}
\text{EMD}(A, B) = \min_{\phi: A \leftrightarrow B} \sum_{a \in A} \| a - \phi(a) \|,
\end{equation}
where $\phi$ denotes a bijection between point clouds $A$ and $B$. While EMD provides a principled transport-based metric, it requires solving an expensive optimization problem that scales poorly with point cloud size ($O(n^3 \log n)$), making it impractical for large-scale applications \cite{rubner2000earth}.

\paragraph{4. Gromov-Hausdorff Distance.}
In \cite{hausdorff2008gesammelte} and \cite{yang2018} comparing two objects reduced to the comparison between the corresponding interpoint distance matrices in the multidimensional scaling. These methods primarily capture rigid similarities and do not account for isometric objects. To address this limitation Facundo Mémoli and Guillermo Sapiro \cite{memoli2004comparing} introduced a method based on the Gromov-Hausdorff distance, which provides a theoretical framework for isometric-invariant recognition.
The Gromov-Hausdorff distance, \(d_{GH}\) measures the similarity between two metric spaces \(X\) and \(Y\) by embedding them into a common metric space \(Z\) and minimizing the Hausdorff distance \(d_H^Z(X, Y)\). The Gromov-Hausdorff distance is defined as
\begin{equation}
d_{GH}(X, Y) = \inf_{Z, f, g} d_H^Z(f(X), g(Y)),
\end{equation}
where, $d_H^Z(f(X), g(Y)) = \max \left( \sup_{x \in f(X)} d_Z(x, g(Y)), \sup_{y \in g(Y)} d_Z(y, f(X)) \right).$
Here, \(f\) and \(g\) are isometric embeddings of \(X\) and \(Y\) into \(Z\), and \(d_Z\) is the intrinsic geodesic distance on \(Z\). Since there is no efficient way to directly compute the Gromov-Hausdorff distance they introduced a metric $d_J$,
\begin{equation}
d_{J}(X,Y) = \underset{\pi\in \mathcal{P}_{n}}{\min} \, \underset{1 \leq i, j \leq n}{\max} \frac{1}{2} \left| d_{X}(x_i, x_j) - d_{Y}(y_{\pi i}, y_{\pi j}) \right|,
\end{equation}

where $\mathcal{P}_{n}$ is the set of all permutations of $\{1,2,...,n\}$, $n$ is the number of points in the point cloud. This metric satisfies $d_{GH}\leq d_{J}$. The metric $d_{J}$ is computable and can be used to replace $d_{GH}$. 

The framework for comparing point clouds developed in \cite{memoli2004comparing} may face challenges with highly complex or higher-dimensional geometric structures due to computational constraints. 
Beyond distance-based comparisons, recent approaches address robust point cloud processing using statistical methods. For instance, Gaussian process regression has been applied to point cloud denoising and outlier detection \cite{Kim2025}, demonstrating the broader landscape of statistical techniques for point cloud analysis.

\subsection{Divergence and Distance Based Measures on GMMs}
Recent works have explored representing point clouds as probability distributions, most commonly using GMMs. Several registration and comparison methods have been developed:

\begin{itemize}
    \item \textbf{GMM-based registration:} Jian and Vemuri \cite{jian2005robust} used Gaussian mixtures with $L^2$ distance for point set registration. Myronenko and Song \cite{song2010} developed the Coherent Point Drift algorithm, treating one point cloud as a GMM and aligning it to another through probabilistic correspondence.
    
    \item \textbf{Divergence-based methods:} Wang et al.~\cite{vemuri2006} introduced cumulative distribution functions with Jensen-Shannon divergence for groupwise registration. Hasanbelliu et al.~\cite{Giraldo2014} employed correntropy and Cauchy-Schwarz divergences. Li et al.~\cite{fuji2021} proposed novel bandwidth estimation strategies for GMM-based comparison.
\end{itemize}

While the above methods use GMMs for registration, they do not exploit the underlying geometric structure of the space of GMMs. Lee et al.~\cite{lee2022statistical} introduced a statistical manifold framework using kernel density estimation, defining parametric probability density functions as statistical representations of point clouds. In contrast, we interpret point clouds as samples from GMMs and endow the space of GMMs with a statistical manifold structure. This enables the use of information geometric tools such as divergences, connections, and curvature—for rigorous shape analysis. Our recent work \cite{vishwakarma2025} further embeds this statistical manifold into the manifold of symmetric positive definite matrices, enriching the geometric structure.

\section{Information Geometry}\label{sec3}
Information geometry studies non-Euclidean geometric structures on spaces of probability distributions, with applications to statistical inference and estimation as well as modern areas such as machine learning and deep learning. Metrics or divergences between probability distributions is important in practical
applications to look at similarity or dissimilarity among the given set of samples. A brief account of the geometry of statistical manifolds is given \cite{Amari2000MethodsOI}, \cite{harsha2014f}. 
For basic knowledge in differential geometry see \cite{lee2012smooth}.

\subsection{Statistical manifold}\label{sec:manifold}
The manifold structure and its geometry for a statistical model are discussed here. Let $(\mathcal{X}, \Sigma, p)$ be a probability space, where $\mathcal{X} \subseteq \mathbb{R}^{n}$. Consider a family $\mathcal{S}$ of probability distributions on $\mathcal{X}$. Suppose each element of $\mathcal{S}$ can be parametrized using $n$ real-valued variables $(\theta^{1},...,\theta^{n})$ so that 
\begin{equation}
\mathcal{S}=\lbrace p_{\theta}=p(x;\theta)\;\mid \;\theta=(\theta^{1},...,\theta^{n}) \in \Theta \rbrace
\end{equation}
where $\Theta$ is an open subset of $\mathbb{R}^{n}$ and the mapping $\theta \mapsto p_{\theta}$ is injective. The family $\mathcal{S}$ is called an $n$-dimensional \textbf{statistical model} or a\textbf{ parametric model}.

For a model $\mathcal{S}=\lbrace p_{\theta} \;\mid \; \theta \in \Theta \rbrace$, the mapping $\varphi:\mathcal{S}\longrightarrow \mathbb{R}^{n}$ defined by $\varphi(p_{\theta})=\theta$ allows us to consider $\varphi=(\theta^{i})$ as a coordinate system for $\mathcal{S}$. Suppose there is a $\mathit{c^{\infty}}$ diffeomorphism $\psi:\Theta\longrightarrow \psi(\Theta)$, where $\psi(\Theta)$ is an open subset of $\mathbb{R}^{n}$. Then, if we use $\rho=\psi(\theta)$ instead of $\theta$ as our parameter, we obtain $\mathcal{S}=\lbrace p_{\psi^{-1}(\rho)}\;\vert \; \rho \in \psi(\Theta) \rbrace$. This expresses the same family of probability distributions $\mathcal{S}=\lbrace p_{\theta} \rbrace$. If parametrizations which are $\mathit{c^{\infty}}$ diffeomorphic to each other is considered to be equivalent then $\mathcal{S}$ is a $\mathit{c^{\infty}}$ differentiable manifold, called the \textbf{statistical manifold} (see Figure ~\ref{fig:manifold})\cite{Amari2000MethodsOI,Amari2013,Nielsen2020}.

\begin{figure}[htbp]
    \centering
    \includegraphics[width=1.0\linewidth]{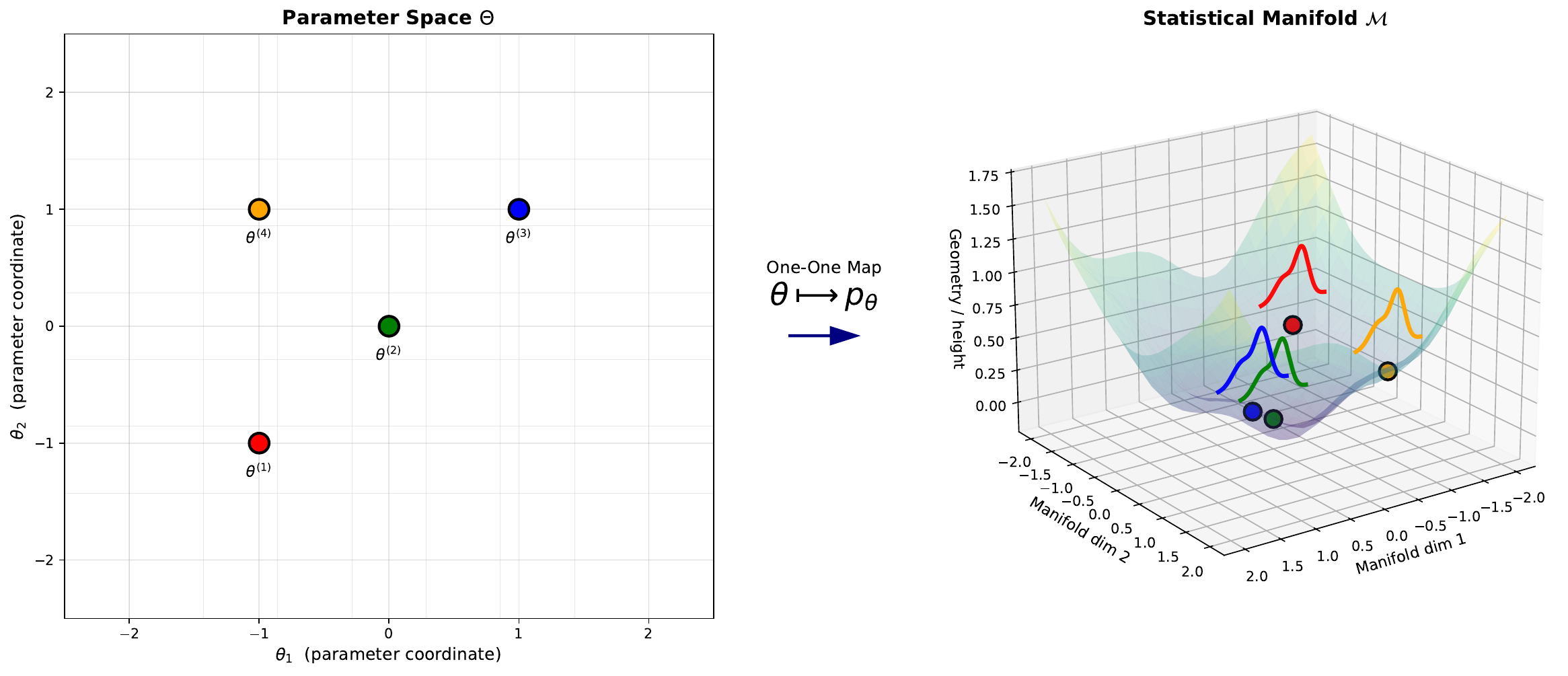}
    \caption{Statistical manifold representation shows the map
between the parameter space and the space of
probability distributions.}
    \label{fig:manifold}
\end{figure}

\subsection{Divergence Measures on Statistical Manifolds}\label{sec:unbounded}

Divergence is a distance-like measure between two points (probability density functions) on a statistical manifold. The \textbf{divergence} $D$ on $S$ is defined as $D=D(.||.):S \times S \to \mathbb{R}$, a smooth function satisfying, for any $p,q\in S$
$$D(p||q) \geq 0\text{ and }D(p||q)=0\text{ iff }p=q.$$

The \textbf{Kullback-Leibler (KL)} divergence is defined as \cite{kullback1951information,cover1991eit},
\begin{equation}
    D_{KL}(p || q) = \int p(x; \theta_1) \log\frac{p(x; \theta_1)}{q(x; \theta_2)} dx
\end{equation}
where $p(x; \theta_1)$ and $q(x; \theta_2)$ are probability density functions. In general, divergences are not symmetric. We modify the KL divergence to a symmetric one called the \textbf{Modified Symmetric KL divergence} $D_{MSKL}(p \parallel q)$, defined as
\begin{equation}
    D_{MSKL}(p \parallel q) = \frac{1}{2}\left[D_{KL}(\sqrt{p} \parallel \sqrt{q}) + D_{KL}(\sqrt{q} \parallel \sqrt{p})\right].
\end{equation}

The well-definedness of the MSKL divergence for GMMs will be established in Section~\ref{sec:bound}, where we derive upper and lower bounds. Note that MSKL differs from the standard symmetric KL divergence \cite{Amari2000MethodsOI} through the square-root transformation applied before symmetrization.

Since no closed-form expression exists for KL divergence between GMMs, we also consider two established approximations as baselines for comparison \cite{1238387,hershey2007approximating}. For GMMs $p(x) = \sum_{i=1}^K \alpha_i p_i(x)$ and $q(x) = \sum_{j=1}^L \beta_j q_j(x)$, where $p_i$ and $q_j$ denote Gaussian components with mixing weights $\alpha_i$ and $\beta_j$, the \textbf{Weighted Average approximation} leverages the convexity of KL divergence:
\begin{equation}
    KL_{WA}(p\|q) = \sum_{i=1}^K \sum_{j=1}^L \alpha_i \beta_j \, D_{KL}(p_i\|q_j).
\end{equation}
While computationally efficient, this provides a lower bound that can significantly underestimate the true divergence when mixture components are well-separated. The \textbf{Matching-Based approximation} assumes dominant contributions arise from the closest components:
\begin{equation}
    KL_{MB}(p\|q) = \sum_{i=1}^K \alpha_i \min_{j=1,\ldots,L}\left[ D_{KL}(p_i\|q_j) + \log\left(\frac{\alpha_i}{\beta_j}\right) \right].
\end{equation}
This provides an upper bound but becomes unreliable when components overlap.

The MSKL divergence offers several advantages over these existing approximations. 
\begin{enumerate}
\item \textbf{Bounded behavior}: As shown in Section~\ref{sec:bound}, MSKL admits upper and lower bounds for GMMs, whereas standard KL divergence can be unbounded.
\item \textbf{Robustness}: The square root transformation reduces sensitivity to outliers and provides better numerical stability, particularly important when comparing noisy point cloud data.
\end{enumerate}

\section{Statistical Manifold Representation for Point Cloud Data}
In this section, we establish a statistical manifold framework for point clouds dataset. Consider the dataset  $X$ with $n$ points in $\mathbb{R}^{m}$, i.e $X=\{x_{1},....,x_{n}\mid x_{i}\in \mathbb{R}^{m}\}$, assume that all the points are distinct. Denote the set of all point clouds by $Z.$

\subsection{Gaussian Mixture Model}
We model the point cloud \(X\subset\mathbb{R}^m\) as i.i.d. samples from a \(K\)-component Gaussian mixture with parameters
\(\Theta=\{(\mu_k,\Sigma_k,\pi_k)_{k=1}^K\}\), where \(\mu_k\in\mathbb{R}^m\), \(\Sigma_k\succ 0\), \(\pi_k\ge 0\), and \(\sum_{k=1}^K \pi_k=1\).
Each component has density
\[
\mathcal N(x;\mu,\Sigma)=\frac{1}{\sqrt{(2\pi)^m|\Sigma|}}
\exp\!\left(-\tfrac12(x-\mu)^\top \Sigma^{-1}(x-\mu)\right),
\]
and the mixture is
\[
p(x;\theta)=\sum_{k=1}^K \pi_k\,\mathcal N(x;\mu_k,\Sigma_k).
\]
The density \(p(\cdot;\theta)\) is the statistical representation of \(X\) \cite{reynolds2009gmm}.

\subsection{Manifold Structure for Point Clouds}

Consider the space $Z$ of point clouds, where each point cloud is represented as $X = \{x_1, \ldots, x_n\mid x_i \in \mathbb{R}^m\}$. Point cloud $X$ has a statistical representation $p(x;\theta)$ where the parameter $\theta$ varies over the parameter space $\Theta=\{\theta=(\mu_1,...,\mu_{K},\Sigma_{1},...,\Sigma_{K},\pi_{1},...,\pi_{K})\mid \mu_{i}\in \mathbb{R}^{m}, $ $ \hspace{0.2cm}\Sigma_{i} \text{ is an } m \times m $ symmetric matrix, $\Sigma_{i=1}^{K}\pi_{i}=1\}$. The set $S =\{p(x;\theta)\mid \theta\in \Theta\}$ representing the space $Z$ of point clouds is the statistical model representing $Z.$ Our aim is to give a geometric structure, called the statistical manifold, to the space $Z.$

\begin{theorem}
Let \small $p(x;\theta_1) =\sum_{i=1}^{K} \alpha_{i} \mathcal{N}(x ; \mu_{i}, \sigma_i^{2})$ and $p(x;\theta_2)=\sum_{j=1}^{L} \beta_{j} \mathcal{N}(x; \nu_{j}, \tau_j^{2})$ be two univariate Gaussian mixture models representing the point cloud $X$ having the number of Gaussian components $K$ and $L$ respectively.  Suppose that for all $x$ in a set with no upper bound,
$$p(x;\theta_1) = p(x;\theta_2).$$ If the parameters in each component of a GMM are distinct and they are ordered lexicographically by variance and mean, then $K=L$ and for each $i$, $\alpha_{i} = \beta_{i}$, $\mu_{i} = \nu_{i}$, and $\sigma_i^{2} = \tau_i^{2}$.
\end{theorem}
\begin{proof}
Arrange the components of each GMM in increasing order of their variances. For $p(x;\theta_1)$ we have $\sigma_1^2 \leq \sigma_2^2 \leq \ldots \leq \sigma_K^2$, and for $p(x;\theta_2)$ we have $\tau_1^2 \leq \tau_2^2 \leq \ldots \leq \tau_L^2$. If $\sigma_i^2 = \sigma_j^2$, then $\mu_i \leq \mu_j$, and similarly if $\tau_i^2 =\tau_j^2$ then $\nu_i\leq \nu_j$.

For all $x$ in a set with no upper bound,

$$\sum_{i=1}^{K} \alpha_{i} \mathcal{N}(x;\mu_{i}, \sigma_{i}^{2}) = \sum_{j=1}^{L} \beta_j \mathcal{N}(x; \nu_{j}, \tau_{j}^{2}). $$

Consider the component with the largest variance in each GMMs. Without loss of generality, assume $\sigma_K^{2}$ and $\tau_L^{2}$ are the largest variances in their respective GMMs. Then,

$$ \lim_{x \to \infty} \frac{\sigma_{K}}{\mathcal{N}(x; \mu_{K}, \sigma_{K}^{2})} \sum_{i=1}^{K} \alpha_{i} \mathcal{N}(x ; \mu_{i}, \sigma_{i}^{2}) = \alpha_{K} $$
$$\lim_{x \to \infty} \frac{\tau_L}{\mathcal{N}(x ; \nu_L, \tau_L^2)} \sum_{j=1}^{L} \beta_j \mathcal{N}(x ; \nu_j, \tau_j^2) = \beta_L $$

Since the two GMMs are equal for all $x$, their limits must also be equal. That means $\alpha_K = \beta_L$ and $(\mu_K, \sigma_K^2) = (\nu_L, \tau_L^2)$. Then,

$$\sum_{i=1}^{K-1} \alpha_{i} \mathcal{N}(x; \mu_{i}, \sigma_{i}^{2}) = \sum_{j=1}^{L-1} \beta_{j} \mathcal{N}(x; \nu_{j}, \tau_{j}^{2}).$$

By repeating the above steps we conclude that $K = L$ and for each $1 \leq i \leq K$, the parameters are equal: $\alpha_{i} = \beta_{i}$, $\mu_{i} = \nu_{i}$, and $\sigma_{i}^{2} = \tau_{i}^{2}$.
\end{proof}
\textbf{Note:} In this paper, we work with the diagonal covariance matrices to reduce the computational complexity.
\begin{theorem}
The parameter space $$\Theta=\{\theta=(\mu_1,...,\mu_{k},\Sigma_{1},...,\Sigma_{K},\alpha_{1},...,\alpha_{K})\mid \mu_{i}\in \mathbb{R}^{m},$$ $$ \Sigma_{i} \text{ is an } m \times m \text{ diagonal covariance matrix, }\Sigma_{i=1}^{K}\alpha_{i}=1\}$$ for GMMs with $K$ components of $m$-dimensional Gaussian is a topological manifold.
\end{theorem}

\begin{proof} For the probability density function
$$ p(x; \theta) = \sum_{i=1}^{K} \alpha_i \mathcal{N}(x ; \mu_i, \Sigma_i),$$
each $\mu_i$ is an $m-$dimensional vector. The $m\times m$ diagonal covariance matrix $ \Sigma_i$ is symmetric, thus having $m$ distinct elements. The total number of mixing coefficients are $K$ and since $\sum_{i=1}^{K}\alpha_{i}=1,$ only $K-1$ are independent. So the total number of independent parameters in $\Theta$ is
$ \text{dim}(\Theta) = K(2m+1)-1.$ Note that  $\Theta$ is an open subset of  $\mathbb{R}^{dim (\Theta)}$, so $\Theta$ is a topological space with standard Euclidean topology. Also, around any point in $\Theta$ there exists a neighborhood that is homeomorphic to an open subset of $\mathbb{R}^{\text{dim}(\Theta)}$. This satisfies the local Euclidean condition for a manifold. Therefore, the parameter space $\Theta$ of m-dimensional GMMs is a topological manifold of dimension $K(2m+1)-1$.
\end{proof}
\begin{theorem}
The statistical model $S$ representing the space $Z$ of point clouds is a statistical manifold of dimension $K(2m+1)-1$.
\end{theorem}
\begin{proof}
Consider the map \( h: \Theta \rightarrow S \), defined by \( h(\theta) = p(x; \theta) \) for \( \theta \in \Theta \). Now to show that this map is injective.
Consider two GMMs with parameter sets \(\theta_1, \theta_2 \in \Theta\)
\[ p(x; \theta_1) = \sum_{i=1}^{K} \alpha_{i1} \mathcal{N}(x ; \mu_{i1}, \Sigma_{i1}) \]
and
\[ p(x; \theta_2) = \sum_{i=1}^{K} \alpha_{i2} \mathcal{N}(x ; \mu_{i2}, \Sigma_{i2}) \]
where \(\Sigma_{i1}\) and \(\Sigma_{i2}\) are diagonal matrices.

Assume \(p(x; \theta_1) = p(x; \theta_2)\) for all \(x\). By the Cramér-Wold theorem\cite{Cramr1936SomeTO}, this equality holds if and only if their projections onto any direction \(\ell \in \mathbb{R}^m\) are equal.

For any direction \(\ell\), the property of multivariate normal distributions ensures that if \(X \sim \mathcal{N}(\mu, \Sigma)\), then
\[ \ell^T X \sim \mathcal{N}(\ell^T\mu, \ell^T\Sigma\ell). \]

Therefore, projecting GMMs gives
\[ \sum_{i=1}^{K} \alpha_{i1} \mathcal{N}(\ell^T x ; \ell^T\mu_{i1}, \ell^T\Sigma_{i1}\ell) = \sum_{i=1}^{K} \alpha_{i2} \mathcal{N}(\ell^T x ; \ell^T\mu_{i2}, \ell^T\Sigma_{i2}\ell) .\]

These projected distributions are univariate GMMs, as each \(\ell^T x\) is scalar-valued. By Theorem 1, the equality of these univariate GMMs for all directions \(\ell\) implies
\[ \alpha_{i1} = \alpha_{i2}, \quad \ell^T\mu_{i1} = \ell^T\mu_{i2}, \quad \ell^T\Sigma_{i1}\ell = \ell^T\Sigma_{i2}\ell \]
for all \(i\) and all \(\ell\). Since this holds for all projection directions \(\ell\), we conclude that \(\theta_1 = \theta_2\), proving that \(h\) is injective.
The parameter space $\Theta$ is a topological manifold and the mapping $h:\Theta\to S$ is injective hence the statistical model representing the space $Z$ of point clouds can be viewed as a statistical manifold of dimension $dim(\Theta)=K(2m+1)-1.$
\end{proof}

\section{MSKL Bounds}\label{sec:bound}
In this section, we establish upper and lower bounds for the MSKL divergence between GMMs. These bounds are the 
theoretical foundation that ensures MSKL is well-defined and 
numerically stable for the GMM comparison, addressing the unboundedness 
issues of existing KL approximations discussed in Section~\ref{sec:unbounded}.

\begin{theorem}
\label{thm:distribution_free_lower}
Let $p,q$ be probability densities on $\mathbb{R}^m$. Then,
\[
D_{\mathrm{MSKL}}(p\|q) \geq S(p)\log S(p)+S(q)\log S(q)-(S(p)+S(q))\log A(p,q).
\]
Where, $S(p) ,S(q)$ and $A(p,q)$ denote $\int_{\mathbb{R}^m}\sqrt{p(x)}\,dx,$ $\int_{\mathbb{R}^m}\sqrt{q(x)}\,dx \text{ and } \int_{\mathbb{R}^{m}}(p(x)q(x))^{1/4}dx$ respectively.
We assume $0<S(p),S(q)<\infty$ and $0<A(p,q)<\infty.$ 

\end{theorem}
\begin{proof}
Denote $r_{p}=\sqrt{p(x})$ and $r_{q}=\sqrt{q(x)}$ and let $\alpha = \frac{r_p}{S(p)}$ and $\beta = \frac{r_q}{S(q)}$, so $\int \alpha = \int \beta = 1$. Then,
\begin{align}
D_{\mathrm{KL}}(r_p\|r_q) &= \int r_p \log\frac{r_p}{r_q} = \int S(p)\alpha \log\frac{S(p)\alpha}{S(q)\beta}\\
&= S(p)\log\frac{S(p)}{S(q)} + S(p) \cdot D_{\mathrm{KL}}(\alpha\|\beta)
\end{align}
Similarly for $D_{\mathrm{KL}}(r_q\|r_p)$. Thus,
\small 
{\begin{equation}
D_{\mathrm{MSKL}}(p\|q) = \frac{1}{2}\left[(S(p)-S(q))\log\frac{S(p)}{S(q)}\right] + \frac{1}{2}\left[S(p) \cdot D_{\mathrm{KL}}(\alpha\|\beta) + S(q) \cdot D_{\mathrm{KL}}(\beta\|\alpha)\right]
\end{equation}}
Using the bound $D_{\mathrm{KL}}(u\|v) \geq -2\log\int \sqrt{u(x)v(x)} dx$ and using
\begin{equation}
\int \sqrt{\alpha \beta} = \frac{\int (pq)^{1/4}}{\sqrt{S(p)S(q)}} = \frac{A(p,q)}{\sqrt{S(p)S(q)}}
\end{equation}
we obtain
\begin{align}
S(p) \cdot D_{\mathrm{KL}}(\alpha\|\beta) &\geq -2S(p)\log\left(\frac{A(p,q)}{\sqrt{S(p)S(q)}}\right)\\
S(q) \cdot D_{\mathrm{KL}}(\beta\|\alpha) &\geq -2S(q)\log\left(\frac{A(p,q)}{\sqrt{S(p)S(q)}}\right)
\end{align}
Substituting into (13) completes the proof.
\end{proof}

\begin{theorem}
\label{thm:S_upper_bound}
Let $p(x)=\sum_{i=1}^K \pi_i \phi_i(x)$ be a GMM with $\phi_i=\mathcal{N}(\mu_i,\Sigma_i)$, $\pi_i\geq 0$, $\sum_{i=1}^K\pi_i=1$. Then
$S(p) \leq \bar{S}(p).$
where, $\bar{S}(p)=(8\pi)^{m/4}\sum_{i=1}^K \sqrt{\pi_i}\,|\Sigma_i|^{1/4}.$
\end{theorem}
\begin{proof}
By concavity of square root,
\begin{equation}
\sqrt{\sum_{i=1}^K \pi_i \phi_i(x)} \leq \sum_{i=1}^K \sqrt{\pi_i} \sqrt{\phi_i(x)}.
\end{equation}
Integrating both sides,
\begin{equation}
S(p) = \int_{\mathbb{R}^m} \sqrt{p(x)} \, dx \leq \sum_{i=1}^K \sqrt{\pi_i} \int_{\mathbb{R}^m} \sqrt{\phi_i(x)} \, dx.
\end{equation}
For $\phi_i = \mathcal{N}(x;\mu_i, \Sigma_i)$,

\small{
\begin{align}
\int_{\mathbb{R}^m} \sqrt{\mathcal{N}(x; \mu_i, \Sigma_i)} \, dx &= \int_{\mathbb{R}^m} \left[\frac{1}{(2\pi)^{m/2}|\Sigma_i|^{1/2}} \exp\left(-\frac{1}{4}(x-\mu_i)^T\Sigma_i^{-1}(x-\mu_i)\right)\right] dx\\
&= \frac{1}{(2\pi)^{m/4}|\Sigma_i|^{1/4}} \int_{\mathbb{R}^m} \exp\left(-\frac{1}{4}y^T\Sigma_i^{-1}y\right) dy \quad (y = x - \mu_i)\\
&= \frac{1}{(2\pi)^{m/4}|\Sigma_i|^{1/4}} \cdot (4\pi)^{m/2}|\Sigma_i|^{1/2} = (8\pi)^{m/4}|\Sigma_i|^{1/4}
\end{align}
}
Substituting into (17) completes the proof.
\end{proof}

\begin{theorem}
\label{thm:A_upper_bound}
Let $p(x)=\sum_{i=1}^K \pi_i\phi_i(x)$ and $q(x)=\sum_{j=1}^L \rho_j\psi_j(x)$ be two GMMs on $\mathbb{R}^m$ with $\phi_i=\mathcal{N}(\mu_i,\Sigma_i)$, $\psi_j=\mathcal{N}(\nu_j,\Lambda_j)$, and mixing weights satisfying $\sum_i\pi_i=\sum_j\rho_j=1$. Then,
\[
A(p,q)\leq \bar{A}(p,q).
\]
Where, $\bar{A}(p,q)$ denotes $\sum_{i=1}^K\sum_{j=1}^L \pi_i^{1/4}\rho_j^{1/4} I_{ij}$ where, $I_{ij} = \int_{\mathbb{R}^m} \phi_i(x)^{1/4}\psi_j(x)^{1/4}\,dx$ can be computed in closed form as
\[
I_{ij} = (2\pi)^{m/4}\frac{|P_{ij}|^{-1/2}}{|\Sigma_i|^{1/8}|\Lambda_j|^{1/8}}\exp\left(-\tfrac{1}{2}Q_{ij}\right),
\]
with $P_{ij} = \tfrac{1}{4}(\Sigma_i^{-1}+\Lambda_j^{-1})$, $h_{ij} = \tfrac{1}{4}(\Sigma_i^{-1}\mu_i+\Lambda_j^{-1}\nu_j)$, $R_{ij} = \tfrac{1}{8}(\mu_i^\top\Sigma_i^{-1}\mu_i + \nu_j^\top\Lambda_j^{-1}\nu_j)$, and $Q_{ij} = 2R_{ij} - h_{ij}^\top P_{ij}^{-1}h_{ij}$.
\end{theorem}

\begin{proof}
Jensen's inequality gives
\[
p(x)^{1/4} \leq \sum_i \pi_i^{1/4}\phi_i(x)^{1/4}, \quad q(x)^{1/4} \leq \sum_j \rho_j^{1/4}\psi_j(x)^{1/4}.
\]
Therefore,
\small{
\[
A(p,q) = \int (p(x)q(x))^{1/4}dx \leq \sum_{i,j} \pi_i^{1/4}\rho_j^{1/4}\int \phi_i(x)^{1/4}\psi_j(x)^{1/4}dx = \sum_{i,j} \pi_i^{1/4}\rho_j^{1/4} I_{ij}.
\]
}

To compute $I_{ij}$, We have
\[
\begin{aligned}
\phi_i(x)^{1/4}\psi_j(x)^{1/4} 
&= \frac{1}{(2\pi)^{m/4}|\Sigma_i|^{1/8}|\Lambda_j|^{1/8}} \\
&\quad \times \exp\!\left(-\tfrac{1}{8}(x-\mu_i)^\top\Sigma_i^{-1}(x-\mu_i)
-\tfrac{1}{8}(x-\nu_j)^\top\Lambda_j^{-1}(x-\nu_j)\right).
\end{aligned}
\]

Expanding the quadratic forms and collecting terms, this equals
\[
\frac{1}{(2\pi)^{m/4}|\Sigma_i|^{1/8}|\Lambda_j|^{1/8}}\exp\left(-\frac{1}{2}x^\top P_{ij} x + x^\top h_{ij} - R_{ij}\right)
\]
where $P_{ij} = \frac{1}{4}(\Sigma_i^{-1}+\Lambda_j^{-1})$, $h_{ij} = \frac{1}{4}(\Sigma_i^{-1}\mu_i+\Lambda_j^{-1}\nu_j)$, and $R_{ij} = \frac{1}{8}(\mu_i^\top\Sigma_i^{-1}\mu_i + \nu_j^\top\Lambda_j^{-1}\nu_j)$.
We have $-\frac{1}{2}x^\top P_{ij} x + x^\top h_{ij} = -\frac{1}{2}(x-P_{ij}^{-1}h_{ij})^\top P_{ij}(x-P_{ij}^{-1}h_{ij}) + \frac{1}{2}h_{ij}^\top P_{ij}^{-1}h_{ij}$.

Substituting this back and integrating we get

\[
I_{ij} = \int_{\mathbb{R}^m} \phi_i(x)^{1/4}\psi_j(x)^{1/4}dx
\]
\[
\begin{aligned}
&= \frac{1}{(2\pi)^{m/4}|\Sigma_i|^{1/8}|\Lambda_j|^{1/8}}
   \exp\!\left(\tfrac{1}{2}h_{ij}^\top P_{ij}^{-1}h_{ij} - R_{ij}\right) \\
&\quad \int_{\mathbb{R}^m} 
   \exp\!\left(-\tfrac{1}{2}(x-P_{ij}^{-1}h_{ij})^\top P_{ij}(x-P_{ij}^{-1}h_{ij})\right)\,dx.
\end{aligned}
\]

The Gaussian integral evaluates to $(2\pi)^{m/2}|P_{ij}|^{-1/2}.$ Then,
\[
I_{ij} = \frac{1}{(2\pi)^{m/4}|\Sigma_i|^{1/8}|\Lambda_j|^{1/8}}\exp\left(\frac{1}{2}h_{ij}^\top P_{ij}^{-1}h_{ij} - R_{ij}\right) \cdot (2\pi)^{m/2}|P_{ij}|^{-1/2}
\]

which gives
\[
I_{ij} 
= (2\pi)^{m/4}\frac{|P_{ij}|^{-1/2}}{|\Sigma_i|^{1/8}|\Lambda_j|^{1/8}}\exp\left(\frac{1}{2}h_{ij}^\top P_{ij}^{-1}h_{ij} - R_{ij}\right).
\]

Since $(2\pi)^{m/4} = 2^{m/4}\pi^{m/4}$ and $|P_{ij}| = 4^{-m}|\Sigma_i^{-1}+\Lambda_j^{-1}|$, we have $|P_{ij}|^{-1/2} = 2^m|\Sigma_i^{-1}+\Lambda_j^{-1}|^{-1/2}$. Let $Q_{ij} = 2R_{ij} - h_{ij}^\top P_{ij}^{-1}h_{ij}$. This gives the stated formula for $I_{ij}$.
\end{proof}

\begin{theorem}
\label{thm:computable_lower_bound}
Let $p,q$ be GMMs. Let $\bar{S}(p),\bar{S}(q)$ be the bounds from Theorem~\ref{thm:S_upper_bound} and let $\bar{A}(p,q)$ be the bound from Theorem~\ref{thm:A_upper_bound}. Then,
\[
D_{\mathrm{MSKL}}(p\|q)\;\ge\max\!\left\{0,\;S(p)^*\log\frac{S(p)^*}{\bar{A}(p,q)}+S(q)^*\log\frac{S(q)^*}{\bar{A}(p,q)}\right\}.
\]

Where,
\[
S(p)^*:=\min\!\left\{\bar{S}(p),\frac{\bar{A}(p,q)}{e}\right\},\qquad
S(q)^*:=\min\!\left\{\bar{S}(q),\frac{\bar{A}(p,q)}{e}\right\}.
\] respectively.

\end{theorem}

\begin{proof}
By Theorem~\ref{thm:distribution_free_lower},
\small{
\begin{equation}\label{eq:base-lb}
D_{\mathrm{MSKL}}(p\|q)\ge
S(p)\log S(p)+S(q)\log S(q)-(S(p)+S(q))\log A.
\end{equation}
}
Let us denote $S(p)\log S(p)+S(q)\log S(q)-(S(p)+S(q))\log A$ by $F(S(p),S(q),A).$ 
We have,
\[
0<S(p)\le \bar{S}(p),\qquad 0<S(q)\le \bar{S}(q),\qquad 0<A(p,q)\le \bar{A}(p,q).
\]

For fixed positive $S(p),S(q)$,
\[
\frac{\partial}{\partial A}F(S(p),S(q),A)=-\frac{S(p)+S(q)}{A}<0\qquad(A>0).
\]
Hence $F(S(p),S(q),A)\ge F(S(p),S(q),a)$ for all $A\in(0,a]$ where $a$ is the number $\bar{A}(p,q)$. Therefore, from \eqref{eq:base-lb},
\begin{equation}\label{eq:reduced-A}
D_{\mathrm{MSKL}}(p\|q)\;\ge\;F(S(p),S(q),a)
= \big[S(p)\log S(p) - S(p)\log a\big]+\big[S(q)\log S(q) - S(q)\log a\big].
\end{equation}

Define, for fixed $a>0$,
\[
g(S):=S\log S - S\log a\qquad (S>0).
\]
Then, $g'(S)=\log S + 1 - \log a$ and $g''(S)=1/S>0$, so $g$ is strictly convex on $(0,\infty)$ with unique stationary point at $g'(S)=0$, i.e.
\[
\log\!\Big(\tfrac{S}{a}\Big)=-1 \;\Longleftrightarrow\; S=\tfrac{a}{e}.
\]
Consequently, $S=\tfrac{a}{e}$ is the global minimizer of $g$ on $(0,\infty)$.

Now for the constraint $0<S\le \bar{S}:$
\begin{itemize}
\item If $\bar{S}\ge a/e$, then $a/e$ is admissible and for all $S\in(0,\bar{S}]$,
$g(S)\ge g(a/e)$.
\item If $\bar{S}< a/e$, then for $S\in(0,\bar{S}]$ we have $S/a<1/e$, hence
$g'(S)=\log(S/a)+1<0$, so $g$ is strictly decreasing on $(0,\bar{S}]$ and
$g(S)\ge g(\bar{S})$.
\end{itemize}
In both the cases,
\[
g(S)\;\ge\;g\!\left(\min\!\left\{\bar{S},\tfrac{a}{e}\right\}\right).
\]
Applying this with $(S,\bar{S})=(S(p),\bar{S}(p))$ and  with $(S,\bar{S})=(S(q),\bar{S}(q))$ gives
\[
\begin{aligned}
S(p)\log S(p) - S(p)\log a
&\;\ge\;
S_p^*\log S_p^* - S_p^*\log a,\\[4pt]
S(q)\log S(q) - S(q)\log a
&\;\ge\;
S(q)^*\log S(q)^* - S(q)^*\log a.
\end{aligned}
\]


Combining these two inequalities with \eqref{eq:reduced-A} gives
\[
D_{\mathrm{MSKL}}(p\|q)\;\ge\; S(p)^*\log\frac{S(p)^*}{a}+S(q)^*\log\frac{S(q)^*}{a}
= S(p)^*\log\frac{S(p)^*}{\bar{A}(p,q)}+S(q)^*\log\frac{S(q)^*}{\bar{A}(p,q)}.
\]

Since $D_{\mathrm{MSKL}}(p\|q)\ge 0$ always, we may write
\[
D_{\mathrm{MSKL}}(p\|q)\;\ge\;\max\!\left\{0,\; S(p)^*\log\frac{S(p)^*}{\bar{A}(p,q)}+S(q)^*\log\frac{S(q)^*}{\bar{A}(p,q)}\right\}.
\]
\end{proof}

\subsection{Upper Bound via the Bhattacharyya Coefficient}

\begin{theorem}
\label{thm:upper_bounds_mskl}
Let $p$ and $q$ be GMMs. Then, for any $\epsilon>0$ there exists a real number $T(p,q)$ depending on the mixture parameters such that
\[
D_{\mathrm{MSKL}}(p\|q)\ \le\ -\frac{\bar S(p)+\bar S(q)}{2}\,\log\big(\max\{T(p,q),\epsilon\}\big).
\]
\end{theorem}

\begin{proof} By Jensen's inequality for the concave function $\log$,
\[
\frac{1}{S(p)}\int r_p\log r_p
\le \log\!\Big(\frac{\int r_p^2}{S(p)}\Big)
=-\log S(p).
\]

\[
\frac{1}{S(p)}\int r_p\log r_q \le \log\!\Big(\frac{\int r_pr_q}{S(p)}\Big).
\]
Let $BC(p,q)$ denote Bhattacharyya coefficient $\int \sqrt{pq}=\int r_pr_q$. 
Then,
\[
D_{\mathrm{KL}}(r_p\|r_q)=\int r_p\log r_p-\int r_p\log r_q
\le -S(p)\log \mathrm{BC}(p,q).
\]
Similarly $D_{\mathrm{KL}}(r_q\|r_p)\le -S(q)\log \mathrm{BC}(p,q)$, hence
\begin{equation}\label{eq:mskl-bc}
D_{\mathrm{MSKL}}(p\|q)=\tfrac12\!\big(D_{\mathrm{KL}}(r_p\|r_q)+D_{\mathrm{KL}}(r_q\|r_p)\big)
\ \le\ -\frac{S(p)+S(q)}{2}\,\log \mathrm{BC}(p,q).
\end{equation}

Now for Gaussian components $\phi_i=\mathcal{N}(\mu_i,\Sigma_i)$ and $\psi_j=\mathcal{N}(\nu_j,\Lambda_j)$,
\[
\mathrm{BC}(\phi_i,\psi_j)
=\frac{|\Sigma_i|^{1/4}|\Lambda_j|^{1/4}}
{\big|\tfrac{\Sigma_i+\Lambda_j}{2}\big|^{1/2}}
\exp\!\bigg(-\frac{1}{8}(\mu_i-\nu_j)^\top
\Big(\tfrac{\Sigma_i+\Lambda_j}{2}\Big)^{-1}(\mu_i-\nu_j)\bigg).
\]

As we have
\[
\sum_{i,j}\sqrt{\pi_i\phi_i(x)\,\rho_j\psi_j(x)}
\le \sqrt{\Big(\sum_{i,j}\pi_i\phi_i(x)\Big)\Big(\sum_{i,j}\rho_j\psi_j(x)\Big)}
=\sqrt{p(x)q(x)}.
\]
Integrating both sides yields
\[
\mathrm{BC}(p,q)\ \ge\ \sum_{i,j}\pi_i\rho_j\,\mathrm{BC}(\phi_i,\psi_j),
\]
which is explicit and computable from the mixture parameters. Let $T(p,q)$ denote the number 
 $\sum_{i,j}\pi_i\rho_j\,\mathrm{BC}(\phi_i,\psi_j)$
So, we have $T(p,q)\le \mathrm{BC}(p,q).$ Using \eqref{eq:mskl-bc}

\begin{equation}
\label{eq:mskl_bound}
    D_{\mathrm{MSKL}}(p\|q)\ \le\ -\frac{S(p)+S(q)}{2}\,\log T(p,q).
\end{equation}

Uisng \eqref{thm:upper_bounds_mskl} and \eqref{eq:mskl_bound} we obtain
\[
D_{\mathrm{MSKL}}(p\|q)\ \le\ -\frac{\bar S(p)+\bar S(q)}{2}\,\log T(p,q).
\]
\end{proof}

Combining the results from Theorems 4-8, we obtain the following bound for the MSKL.
\small
\begin{equation}
\max\!\left(0, S(p)^* \log \frac{S(p)^*}{\bar{A}(p,q)} + S(q)^* \log \frac{S(q)^*}{\bar{A}(p,q)}\right) \!\leq\! D_{\mathrm{MSKL}}(p\|q) \!\leq\! -\frac{\bar{S}(p) + \bar{S}(q)}{2} \log \max\{T(p,q), \epsilon\}
\end{equation}
These bound is explicitly computable from the GMM parameters and ensure $0 \leq D_{\mathrm{MSKL}}(p\|q) < \infty$ for all GMMs.

While these results provide theoretical guarantee, in practice the exact computation of $D_{\mathrm{MSKL}}$ is often intractable. Instead, empirical evaluation requires discretized approximations and careful dataset preprocessing. In our implementation, we employ a grid-based discretization where we normalize the square roots of the densities for numerical stability. Specifically, we compute
\begin{equation}
D_{\mathrm{MSKL}}^{\mathrm{discrete}}(p||q) = \frac{1}{2} \sum_{i=1}^{N} \left[ \tilde{p}_i \log \frac{\tilde{p}_i}{\tilde{q}_i} + \tilde{q}_i \log \frac{\tilde{q}_i}{\tilde{p}_i} \right]
\end{equation}
where $\tilde{p}_i = \frac{\sqrt{p(g_i) + \epsilon}}{\sum_j \sqrt{p(g_j) + \epsilon}}$ and $\tilde{q}_i = \frac{\sqrt{q(g_i) + \epsilon}}{\sum_j \sqrt{q(g_j) + \epsilon}}$ are the normalized density values at grid points $\{g_i\}_{i=1}^{N}$. This normalized variant provides a scale-invariant version of MSKL that inherits the theoretical properties while offering improved numerical stability. In the following section, we describe the datasets and preprocessing steps used to validate our theoretical findings in numerical experiments.

\section{Data Overview and Preprocessing}
This section discuss an overview of the datasets used in the experimental analysis, namely the MPI FAUST dataset and the G-PCD dataset.

\subsection{Human and Animal Datasets}
For the human dataset, we employ the MPI-FAUST dataset \cite{mpi_faust_dataset} (see Figure~\ref{fig:Human}), which provides high-resolution 3D scans of human figures captured across diverse postural configurations. These point clouds exhibit complex surface topology making them ideal for evaluating shape similarity measures under non-rigid transformations. 

\begin{figure}[!htbp]
    \centering
    \includegraphics[width=0.8\linewidth]{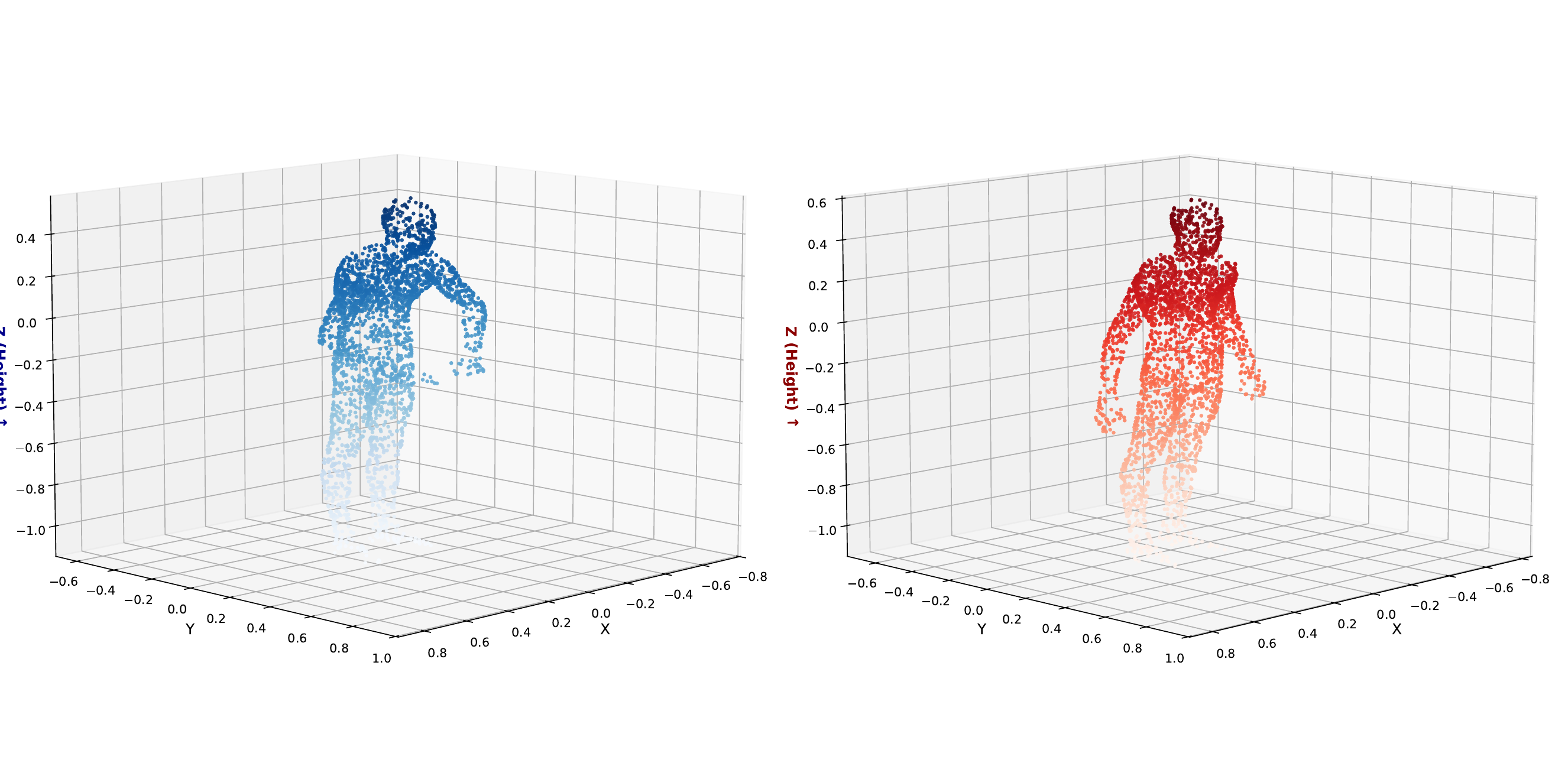}
    \caption{MPI-FAUST human dataset samples showing two distinct poses of a Human body scans.}
    \label{fig:Human}
\end{figure}

For the animal dataset, we utilize the G-PCD dataset \cite{epfl_geometry_point_cloud_dataset} (see Figure~\ref{fig:Animal}), which contains 3D scans of various animal species with various morphological structures. These point clouds feature rich local geometric variations, including sharp features, smooth surfaces, and complex curvature patterns.

\begin{figure}[!htbp]
    \centering
    \includegraphics[width=0.8\linewidth]{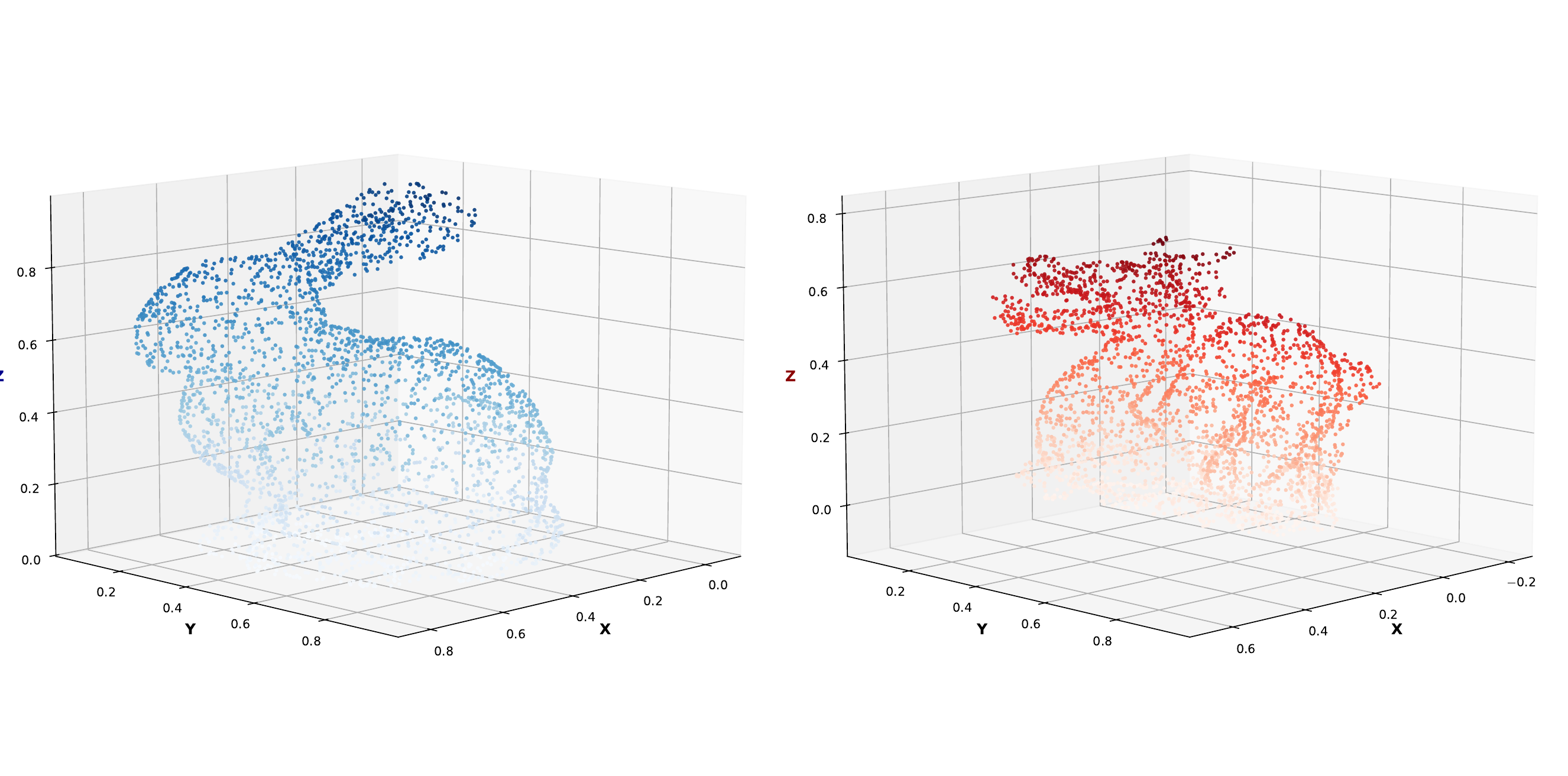}
    \caption{G-PCD animal dataset samples illustrating morphological diversity. Left: Rabbit Point Cloud 1. Right: Dragon Point Cloud 2.}
    \label{fig:Animal}
\end{figure}

\section{Experimental Methodology}

This section describes the complete computational pipeline for measuring shape similarity using the statistical manifold of GMMs. Our approach transforms 3D point clouds into probability distributions on a low-dimensional statistical manifold, where divergence measures quantify shape dissimilarity. The methodology consists of five main stages: preprocessing, geometric feature extraction, manifold embedding, probabilistic modeling, and divergence and metric computation (see Figure~\ref{fig:stage1-5}).

\subsection{Data Preprocessing}

Before feature extraction and manifold-based modeling, each point cloud undergoes a preprocessing step which is designed to establish geometric consistency and ensure compatibility with the computational framework. The raw point clouds often contain several thousand points with non-uniform sampling density across their surfaces. To obtain a uniform resolution while preserving the overall geometric structure, we sample from each point cloud  using the Farthest Point Sampling (FPS) \cite{moenning2003fast}.

FPS progressively selects points that maximize their minimal distance to the current sample set, thereby ensuring broad coverage of the underlying shape. Starting from an initial point chosen uniformly at random, the algorithm iteratively constructs the subset by selecting
\[
p_i = \arg\max_{p \in X \setminus U} \;\min_{j < i} \|p - p_j\|,
\]
which guarantees that each new point is maximally separated from the previously selected ones. Uisng this approach each point cloud is downsampled to 1024 points. To ensure consistency across all scans, each point cloud is then normalized independently. This removes global translation and scale differences, placing all shapes into a comparable coordinate frame without altering their intrinsic geometry.


\subsection{Geometric Feature Extraction}

Following preprocessing, we characterize the local geometry of each point cloud by computing pointwise descriptors derived from neighborhood structure. For every point, we identify its $k=20$ nearest neighbors and compute a set of geometric features based on eigenvalue decomposition of the local covariance matrix.

For each point $\mathbf{x}_i$, let $\mathcal{N}_i = \{\mathbf{x}_{j_1}, \ldots, \mathbf{x}_{j_k}\}$ denote its $k$ nearest neighbors. The covariance matrix of the centered neighborhood is given by
\begin{equation}
\mathbf{C}_i = \frac{1}{k} \sum_{\mathbf{x}_j \in \mathcal{N}_i} (\mathbf{x}_j - \bar{\mathbf{x}}_i)(\mathbf{x}_j - \bar{\mathbf{x}}_i)^\top,
\end{equation}
where $\bar{\mathbf{x}}_i$ is the neighborhood centroid. The eigendecomposition yields eigenvalues $\lambda_1 \geq \lambda_2 \geq \lambda_3 \geq 0$ and corresponding eigenvectors. These eigenvalues characterize the local geometric structure.

From this eigenstructure, we extract eleven geometric features $\mathbf{f}_i \in \mathbb{R}^{11}$. Let $\Sigma_\lambda = \sum_{j=1}^3 \lambda_j$ denote the sum of eigenvalues. The features include: eigenvalue sum ($\Sigma_\lambda$), omnivariance ($(\prod_{j=1}^3 \lambda_j)^{1/3}$), and eigenentropy ($-\sum_{j=1}^3 p_j \ln p_j$ where $p_j = \lambda_j / \Sigma_\lambda$). We compute three shape factors normalized by the eigenvalue sum: anisotropy ($(\lambda_1 - \lambda_3)/\Sigma_\lambda$), planarity ($(\lambda_2 - \lambda_3)/\Sigma_\lambda$), and linearity ($(\lambda_1 - \lambda_2)/\Sigma_\lambda$). Additionally, we include the three normalized eigenvalues ($\lambda_j / \Sigma_\lambda$ for $j=1,2,3$, where the third component represents surface variation), sphericity ($\lambda_3/\lambda_1$), and verticality ($1 - |v_z|$, where $v_z$ is the $z$-component of the eigenvector corresponding to $\lambda_3$) \cite{hackel2016contour}.


\subsection{Manifold Embedding via Isomap}

Following preprocessing and geometric feature extraction, each point is represented jointly by its normalized coordinates and local descriptors. For a point in a cloud, we denote by $\mathbf{x}_i \in \mathbb{R}^3$ its normalized 3D position and by $\mathbf{f}_i \in \mathbb{R}^{11}$ the corresponding geometric feature vector. These are concatenated into a single augmented descriptor
\begin{equation}
\mathbf{y}_i = 
\begin{bmatrix}
\mathbf{x}_i \\
\mathbf{f}_i
\end{bmatrix}
\in \mathbb{R}^{14},
\end{equation}
so that the point cloud is represented as a set of 14-dimensional vectors. After constructing these descriptors for both point clouds, all points are stacked into a single matrix $\mathbf{Y} \in \mathbb{R}^{(n_1 + n_2)\times 14}$, which serves as the input to the manifold learning stage.

We use Isomap \cite{teenbaum2000global} to embed this high-dimensional augmented representation into a low-dimensional latent space. Isomap preserves geodesic distances on the underlying manifold by approximating them via shortest paths in a $k$-nearest neighbor graph and then applying multidimensional scaling. In this setting, the embedding map
\begin{equation}
\mathcal{F} : \mathbb{R}^{14} \rightarrow \mathbb{R}^2
\end{equation}
is learned from $\mathbf{Y}$, producing two-dimensional latent coordinates $\mathbf{z}_i = \mathcal{F}(\mathbf{y}_i)$ for each point.

A key design choice is to embed both point clouds jointly rather than separately. If the two preprocessed and feature-augmented clouds contain $\{\mathbf{y}_i^{(1)}\}_{i=1}^{n_1}$ and $\{\mathbf{y}_j^{(2)}\}_{j=1}^{n_2}$, respectively, we form the combined dataset
\begin{equation}
\mathcal{Y}_{\text{joint}} = \{\mathbf{y}_i^{(1)}\}_{i=1}^{n_1} \cup \{\mathbf{y}_j^{(2)}\}_{j=1}^{n_2},
\end{equation}
and apply Isomap once to $\mathcal{Y}_{\text{joint}}$. The resulting latent coordinates
\[
\{\mathbf{z}_i^{(1)}\}_{i=1}^{n_1}, \qquad \{\mathbf{z}_j^{(2)}\}_{j=1}^{n_2} \subset \mathbb{R}^2
\]
thus lie in a common coordinate system, this ensure that all the subsequent computations are performed on the same latent space.


\subsection{Statistical Manifold Modeling}

With the two point clouds embedded in a common two-dimensional latent space, we model each shape using a Gaussian Mixture Model (GMM). The theoretical formulation and statistical manifold structure of GMMs are established in Section~\ref{sec:manifold}; here we describe the practical implementation.

To determine the number of mixture components $K$, we evaluate the Bayesian Information Criterion (BIC) for $K \in \{1, 2, \ldots, 10\}$ and select
\begin{equation}
K^{*} = \arg\min_{K} \mathrm{BIC}(K).
\end{equation}
The GMM parameters $\Theta = \{(\boldsymbol{\mu}_k, \Sigma_k, \pi_k)\}_{k=1}^{K^*}$ are then estimated using the EM algorithm. For computational efficiency, we employ diagonal covariance matrices
\begin{equation}
\Sigma_k = \mathrm{diag}(\sigma_{k1}^2,\, \sigma_{k2}^2),
\end{equation}
matching the two-dimensional latent representation. A regularization term of $10^{-6}$ is added to each diagonal element for numerical stability, and EM is terminated after at most 200 iterations. The resulting GMMs provide smooth probabilistic representations $p(\mathbf{z}; \Theta_1)$ and $q(\mathbf{z}; \Theta_2)$ for the two point clouds on the latent manifold, which are used directly for divergence computation.

\begin{figure}
    \centering
    \includegraphics[width=1.0\linewidth]{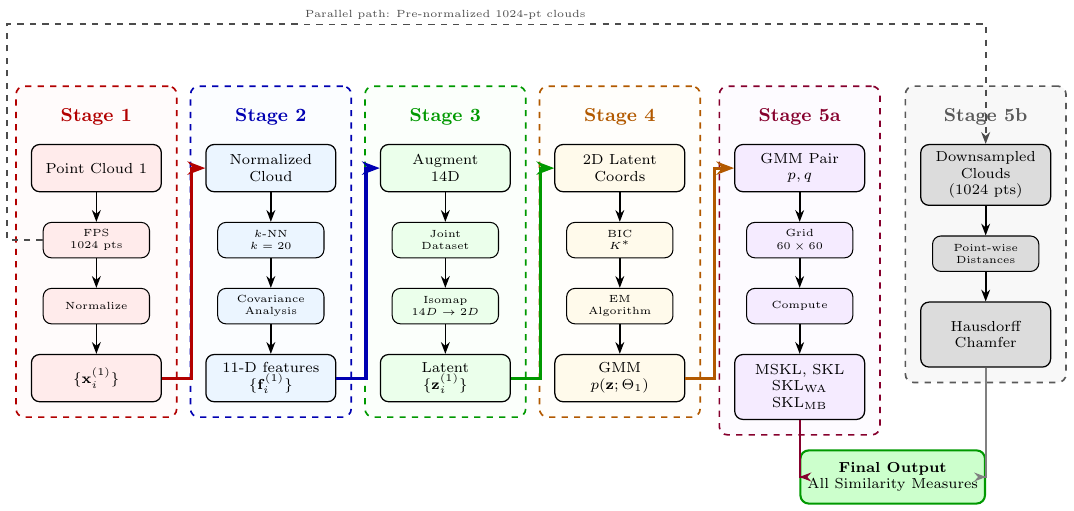}
    \caption{Five-stage computational pipeline for point cloud shape analysis on the statistical manifold of GMMs.}
    \label{fig:stage1-5}
\end{figure}
\subsection{Divergence Computation}
 
The divergence measures are evaluated on the GMMs defined over the 2D latent space. We approximate them via a regular grid discretization. For a given pair of latent embeddings, first compute the bound range $[\mathbf{z}_{\min}, \mathbf{z}_{\max}]$ enclosing all points, expand it by a $5\%$ margin in each direction,
\[
\mathbf{z}_{\min} \leftarrow \mathbf{z}_{\min} - 0.05(\mathbf{z}_{\max} - \mathbf{z}_{\min}), 
\qquad
\mathbf{z}_{\max} \leftarrow \mathbf{z}_{\max} + 0.05(\mathbf{z}_{\max} - \mathbf{z}_{\min}),
\]
and then construct a uniform grid of resolution $60 \times 60$ over this domain. This yields $N = 3600$ evaluation points $\{\mathbf{g}_i\}_{i=1}^{N}$. The corresponding cell area $\Delta A$ is used when required for mass-normalization.

\paragraph{1. Modified Symmetric KL Divergence.}
We approximate the MSKL divergence $D_{\text{MSKL}}(p\|q)$ by evaluating the GMM densities $p(\mathbf{g}_i)$ and $q(\mathbf{g}_i)$ at all grid points, applying the square-root transformation, and normalizing to obtain discrete probability distributions $\{\tilde{p}_i\}_{i=1}^{N}$ and $\{\tilde{q}_i\}_{i=1}^{N}$. The discrete MSKL is then computed as
\begin{equation}
D_{\text{MSKL}}^{\text{disc}}(p\|q) = \frac{1}{2} \sum_{i=1}^{N} 
\left[\tilde{p}_i \log\frac{\tilde{p}_i}{\tilde{q}_i} + \tilde{q}_i \log\frac{\tilde{q}_i}{\tilde{p}_i}\right],
\end{equation}
with regularization $\epsilon = 10^{-8}$ added before normalization to ensure numerical stability.

\paragraph{2. Baseline Distances.}
For comparison, Hausdorff and Chamfer distances are computed on the original downsampled point clouds, using the definitions from Section~\ref{sec:related}. These classical geometric distances provide complementary baselines that operate directly on Euclidean coordinates.

\paragraph{3. Baseline KL Divergences Between GMMs.}
We also evaluate the symmetric KL divergence (SKL) along with the symmetric weighted Average (SKL$_{\text{WA}}$) and symmetric matching-Based (SKL$_{\text{MB}}$) approximations defined in Section~\ref{sec:unbounded}, all computed on the same latent-space GMMs. 

\section{Experimental Results}\label{sec:results}

The proposed geometric method is applied to two different cases and it is compared with other traditional measures. Each case demonstrates the effectiveness of the proposed method for the point cloud comparison.

\subsection{Human Pose Discrimination}

We analyze four scans from a single subject in the MPI-FAUST dataset, representing various postures in which three similar standing poses with minor arm variations (arms down, hands on hips, arms extended) and one walking pose. Table~\ref{tab:faust_results} presents the pairwise similarity measures for all combinations.

\begin{table}[!htbp]
\centering
\caption{Pairwise similarity measures for MPI-FAUST human body scans (same subject, different poses). Lower values indicate greater similarity.}
\label{tab:faust_results}
\begin{tabular}{m{2.5cm}cccccc}
\hline
\textbf{Comparison} & \textbf{MSKL} & \textbf{SKL} & \textbf{SKL$_{\text{WA}}$} & \textbf{SKL$_{\text{MB}}$} & \textbf{Hausdorff} & \textbf{Chamfer} \\
\hline
\raisebox{-.5\height}{\includegraphics[width=1.2cm]{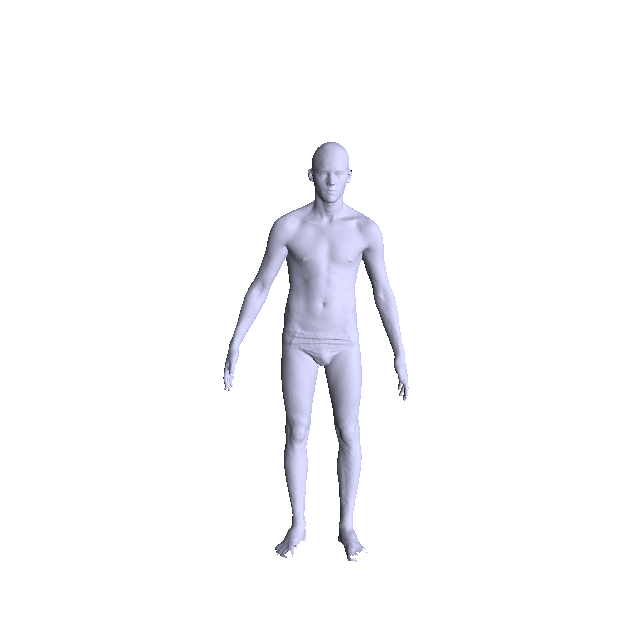}}
\raisebox{-.5\height}{\includegraphics[width=1.2cm]{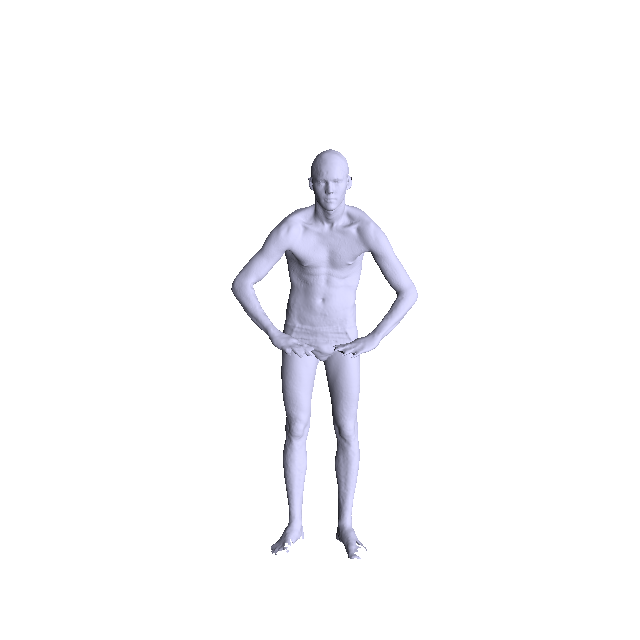}}
& 0.0333 & 0.1023 & 3.056 & 0.6835 & 0.2616 & 0.0571 \\[0.3cm]
\hline
\raisebox{-.5\height}{\includegraphics[width=1.2cm]{tr_scan_000.png}}
\raisebox{-.5\height}{\includegraphics[width=1.2cm]{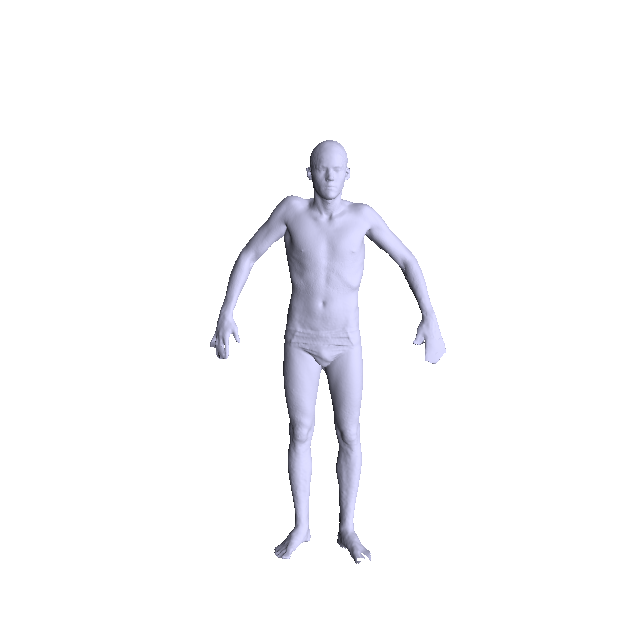}}
& 0.0004 & 0.0012 & 2.531 & 0.0014 & 0.2000 & 0.0340 \\[0.3cm]
\hline
\raisebox{-.5\height}{\includegraphics[width=1.2cm]{tr_scan_001.png}}
\raisebox{-.5\height}{\includegraphics[width=1.2cm]{tr_scan_002.png}}
& 0.1065 & 0.2158 & 4.901 & 1.627 & 0.2616 & 0.0517 \\[0.3cm]
\hline
\raisebox{-.5\height}{\includegraphics[width=1.2cm]{tr_scan_000.png}}
\raisebox{-.5\height}{\includegraphics[width=1.2cm]{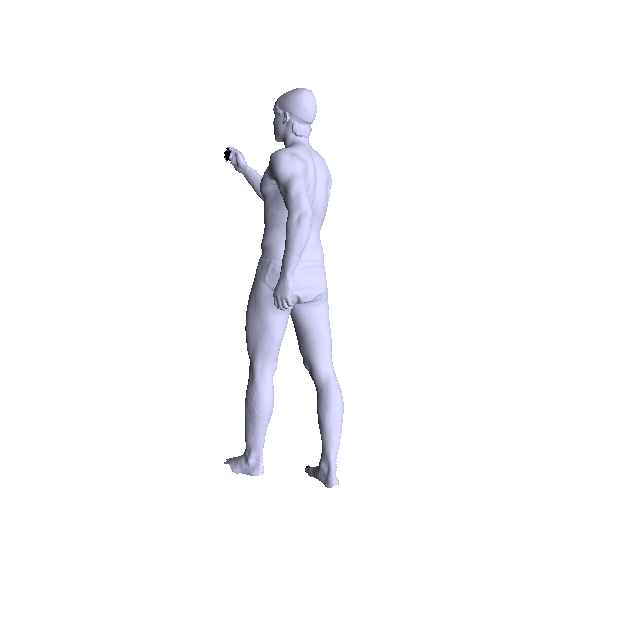}}
& 0.1105 & 0.1778 & 1.502 & 0.4361 & 0.8803 & 0.2127 \\[0.3cm]
\hline
\raisebox{-.5\height}{\includegraphics[width=1.2cm]{tr_scan_001.png}}
\raisebox{-.5\height}{\includegraphics[width=1.2cm]{tr_scan_098.png}}
& 0.0094 & 0.0268 & 2.017 & 0.0766 & 0.4208 & 0.1207 \\[0.3cm]
\hline
\end{tabular}
\end{table}

MSKL divergence demonstrates stable, monotonic behavior with values ranging from 0.0004 (nearly identical standing poses) to 0.1105 (standing versus walking), providing clear discrimination that reflects the postural variation. In contrast, Hausdorff distance exhibits extreme sensitivity also showing poor discrimination for similar poses (0.20--0.26 for most standing comparisons). Chamfer distance offers better stability (0.034--0.213).

The KL-based approximations show major problems. SKL$_{\text{WA}}$ ranges from 1.5--4.9 that fails to distinguish different poses well, while SKL$_{\text{MB}}$ shows unstable behavior with values changing from 0.0014 to 1.627 for similarly posed standing configurations. The SKL shows reasonable ability to distinguish poses 0.0012--0.2158 but has no guaranteed limits for numerical stability and boundedness.

\subsection{Animal Data Comparison}

To evaluate divergence behavior on shapes with fundamentally different global topologies, we compare the Rabbit and dragon models from the G-PCD dataset. Table~\ref{tab:cross_category} presents the results.
\begin{table}[h!]
\centering
\caption{Similarity measures for cross-category comparison (bunny vs dragon from G-PCD dataset).}
\label{tab:cross_category}
\begin{tabular}{cc*{6}{c}}
\hline
\textbf{Bunny} & \textbf{Dragon} & \textbf{MSKL} & \textbf{SKL} & \textbf{SKL$_{\text{WA}}$} & \textbf{SKL$_{\text{MB}}$} & \textbf{Hausdorff} & \textbf{Chamfer} \\
\hline
\raisebox{-0.5\height}{\includegraphics[width=2cm]{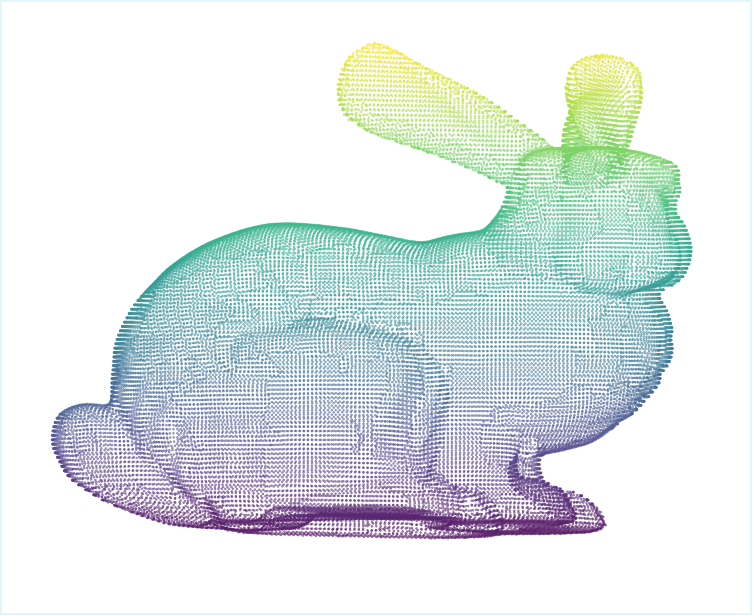}} & 
\raisebox{-0.5\height}{\includegraphics[width=2cm]{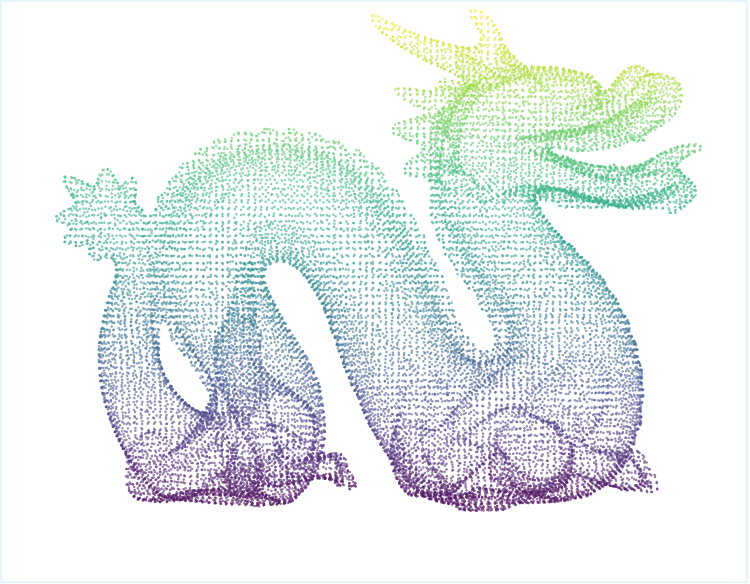}} & 
0.2100 & 0.4207 & 28.139 & 1.3163 & 0.4569 & 0.1411 \\
\hline
\end{tabular}
\end{table}

MSKL yields 0.2100 for these morphologically distinct animals. Among manifold-based measures, MSKL demonstrates bounded numerical behavior: while SKL$_{\text{WA}}$ produces an extreme value of 28.139 MSKL remains within a stable range due to its theoretical upper and lower bounds. SKL produces 0.4207 and SKL$_{\text{MB}}$ yields 1.3163, both operating without guaranteed bounds. Classical distances Hausdorff (0.4569) and Chamfer (0.1411) provide point-wise proximity measures but lack the statistical manifold structure that captures distributional shape properties.

\section{Conclusion and future work}

In this paper, the point clouds are interpreted as the samples from the underlying probability distribution-GMM. Then, the statistical manifold structure is given to the space of point clouds which enables to employ geometric invariants like length, angle, divergence measure, connection and curvature for the analysis of the point clouds. In the present work, we use only divergence measures for the comparison of manifolds of point clouds. Modified Symmetric KL divergence is used for comparison of manifolds of point clouds. The advantages of the proposed method are (i) it offers a mathematically rigorous framework based on information geometry, (ii) it is a comprehensive method that can be used for different types of datasets and (iii) the method's balanced sensitivity and robustness to variations in geometry make it suitable for many practical applications.

Another key theoretical contribution of this work is the derivation of computable upper and lower bounds for the MSKL divergence. These bounds make the divergence theoretically well-posed and allow for practical computation without sacrificing mathematical rigor.

The proposed geometric method consistently exhibits a balanced sensitivity to shape variations and geometry in all the different types of datasets showing its ability to capture the geometry of complex datasets. In the context of 3D shape analysis, the proposed geometric method is effective in differentiating topologically similar shapes, such as human body scans in different postures. Thus, the proposed method provides a new and effective way to compare the manifold of point clouds which could be very useful in the areas like computer vision, 3D modeling, pattern recognition, Biomedical research etc.

The current framework has certain limitations that require future investigation. First, the computational complexity of the proposed framework increases significantly with point cloud size and the number of mixture components, particularly during the EM algorithm and divergence computation on high-resolution grids. While the use of diagonal covariance matrices reduces computational cost, it may not fully capture complex local geometric correlations. Second, the choice of dimensionality for the Isomap embedding is currently fixed at two dimensions for computational efficiency, though higher-dimensional latent representations might preserve more intricate geometric structure for certain shape categories.

The statistical manifold framework can be extended to other geometric invariants beyond divergence measures, including the Riemannian metric, geodesic distances on the manifold, and sectional curvature computations etc. These geometric invariants could provide shape descriptors that capture different aspects of shape similarity.

In conclusion, this work establishes an information geometric framework for point cloud analysis that combines mathematical framework with practical effectiveness. By representing point clouds as points on a statistical manifold of GMMs and introducing the bounded MSKL divergence, we provide a stable and discriminative measure for shape comparison that outperforms existing approximations and baseline metrics. The framework opens new possibilities for using information geometry in 3D shape analysis, with potential applications spanning computer vision, medical imaging, robotics.

\section*{Acknowledgements}
Amit Vishwakarma is thankful to the Indian Institute of Space Science and Technology, Department of Space, Government of India for the award of the doctoral research fellowship.

\section*{Declarations}
\subsection*{Funding}
Not applicable.

\subsection*{Conflict of interest}
The authors declare that they have no known competing financial interests or personal relationships that could have appeared to influence the work reported in this paper.

\subsection*{Ethics approval and consent to participate}
Not applicable.

\subsection*{Data availability}
The study utilized several publicly available datasets for the analysis of point clouds. The MPI FAUST dataset containing human body shapes can be accessed through \url{https://faust-leaderboard.is.tuebingen.mpg.de}. For analyzing animal shapes, we used the EPFL Geometry Point Cloud Dataset (G-PCD) which is available at \url{https://www.epfl.ch/labs/mmspg/downloads/geometry-point-cloud-dataset/}.

\subsection*{Code availability}
The code used in this study will be made available upon reasonable request.

\subsection*{Author contribution}
Both authors contributed equally to this work.



\end{document}